\documentclass[11pt]{article}

\usepackage{biblatex}

\usepackage{microtype}
\usepackage{subfigure}
\usepackage{amsmath, amssymb, amsthm}
\usepackage{bm}
\usepackage{amsfonts}
\usepackage{xcolor}

\usepackage{graphicx}
\graphicspath{{figs/}}
\usepackage{mathtools}

\usepackage{algorithm}
\usepackage{algorithmic}
\usepackage[verbose=true,letterpaper]{geometry}

\usepackage{hyperref}
\usepackage{cleveref}

\newtheorem{theorem}{Theorem}
\newtheorem{othertheorem}{othertheorem}[section]
\newtheorem{lemma}[othertheorem]{Lemma}
\newtheorem{corollary}[othertheorem]{Corollary}
\newtheorem{proposition}[othertheorem]{Proposition}

\newtheorem{fact}[othertheorem]{Fact}
\newtheorem{rem}[othertheorem]{Remark}

\def\R{\mathbb{R}}

\numberwithin{equation}{section}

\title{A Dynamical View on Optimization Algorithms of Overparameterized Neural Networks}

\author{Zhiqi Bu\thanks{Department of Applied Mathematics and Computational Science, University of Pennsylvania. Email: {\tt zbu@sas.upenn.edu} } 
 \and Shiyun Xu\thanks{Department of Applied Mathematics and Computational Science, University of Pennsylvania. Email: {\tt shiyunxu@sas.upenn.edu} } 
 \and Kan Chen\thanks{Department of Applied Mathematics and Computational Science, University of Pennsylvania. Email: {\tt kanchen@sas.upenn.edu} } }

\newcommand{\m}{\bm{m}}
\renewcommand{\v}{\bm{v}}

\newcommand{\w}{\bm{w}}

\newcommand{\x}{\bm{x}}

\newcommand{\y}{\bm{y}}
\newcommand{\W}{\bm{W}}

\newcommand{\Q}{\bm{Q}}
\renewcommand{\a}{\bm{a}}
\renewcommand{\H}{\bm{H}}
\newcommand{\I}{\mathbb{I}}
\newcommand{\f}{\bm{f}}

\date{}

\begin{document}

\maketitle
\begin{abstract}
 When equipped with efficient optimization algorithms, the over-parameterized neural networks have demonstrated high level of performance even though the loss function is non-convex and non-smooth. While many works have been focusing on understanding the loss dynamics by training neural networks with the gradient descent (GD), in this work, we consider a broad class of optimization algorithms that are commonly used in practice. For example, we show from a dynamical system perspective that the Heavy Ball (HB) method can converge to global minimum on mean squared error (MSE) at a linear rate (similar to GD); however, the Nesterov accelerated gradient descent (NAG) may only converge to global minimum sublinearly.
  
  Our results rely on the connection between neural tangent kernel (NTK) and finitely-wide over-parameterized neural networks with ReLU activation, which leads to analyzing the limiting ordinary differential equations (ODE) for optimization algorithms. We show that, optimizing the non-convex loss over the weights corresponds to optimizing some strongly convex loss over the prediction error. As a consequence, we can leverage the classical convex optimization theory to understand the convergence behavior of neural networks. We believe our approach can also be extended to other optimization algorithms and network architectures.
\end{abstract}

\section{Introduction}
Neural Tangent Kernel (NTK) \cite{jacot2018neural} has taken a huge step in understanding the behaviors of over-parameterized neural networks: it has been shown that the training process of a neural network can be characterized by a kernel matrix. Although the NTK matrix is randomly initialized and time-dependent, it in fact stays close to a constant limiting kernel matrix, on which the analysis of neural networks renders much easier. Leveraging NTK, researches have extensively studied how different aspects affect the convergence of training neural networks. For example, weight initialization with large variance can accelerate the convergence but worsen the generalization ability of neural networks \cite{dugradient,xiaodisentangling}. It has been analyzed in \cite{du2018gradient} that a two-layer fully-connected neural network (FCNN) with ReLU activation provably and globally converges to zero training loss. Later the results are enriched by extending the global convergence to multi-layer (deep) FCNN, convolutional neural networks (CNN), recurrent neural networks (RNN) and residual neural networks (ResNet) \cite{du2018gradient,allen2019convergence,arora2019exact, allen2019convergence0,zou2019improved,zou2020gradient}. Especially, the comprehensive result in \cite{allen2019convergence} covers all above-mentioned network architectures, as well as different losses such as mean square error (MSE) and cross-entropy. Specifically, the over-parameterized neural networks enjoy an exponentially decaying MSE and a polynomially decaying cross-entropy. Other examples studied the convergence with NTK under different input distribution \cite{du2018gradient,allen2019convergence}, activation functions \cite{du2018power} and normalization layers \cite{dukler2020optimization}.

There are also a wealth of recent literature on the generalization properties of the wide neural networks based on NTK. For example, using a data-dependent Rademacher complexity measure, a generalization bound independent of network size for a two-layer, ReLU activated, FCNN can be obtained \cite{neyshabur2018role, arora2019fine,allen2019learning}. Several lines of work \cite{bartlett2017spectrally, neyshabur2017pac,cisse2017parseval} gave spectrally-normalized margin-based generalization bound to explain the nice generalization phenomenon of over-parameterized networks. The others \cite{dziugaite2017computing, neyshabur2017pac, li2018learning} derived the generalization bound from the view of PAC-Bayes and compression approaches. More examples include
the generalization bound for high dimension data \cite{advani2020high, ghorbani2019linearized} and with weight decay \cite{lee2020finite}. 

While many works have been focused on the convergence theory of neural network architectures, one would argue that the efficient optimization of a neural network is as important as the design of neural networks. Nevertheless, to our knowledge, this is the first paper to go beyond GD (and SGD) and analyze the global convergence of neural networks from a more general optimization algorithm perspective. In this work, we employ some state-of-the-art optimizers and analyze their convergence behavior under the NTK regime. A related work is \cite{lee2019wide}, where authors also study GD with momentum under the NTK regime but only empirically. 
\\\\
We seek to answer the following questions:
\begin{itemize}
\item Can different optimization algorithms besides GD provably find global minimum in the non-convex optimization of neural networks?
\item Can we leverage the classic convex optimization theory to explain the behavior (e.g. acceleration and rate) of optimization algorithms when training the neural networks?
\end{itemize}

Intriguingly, the known results of global and linear convergence of GD (and SGD) on MSE of neural networks is somewhat surprising, since losses are non-convex with respect to weights. Note that the traditional convex optimization theory only claims $O(1/t)$ rate of GD, even for convex losses, and we only achieve the linear convergence rate when the loss is strongly convex. This phenomenon suggests a close relationship between the non-convex optimization in neural networks and some convex optimization problems that we will elaborate in this paper. We establish such connection rigorously which allows us to conveniently employ the classic convex optimization theory at low cost, thus bridging both worlds to inspire new insights into the training of neural networks.

Once we view the evolution of neural networks during training as an ordinary differential equation (ODE), a.k.a. the gradient flow, there are plenty of fruitful results in the long history of convex optimization \cite{polyak1987introduction, boyd2004convex, nocedal2006numerical, ruszczynski2011nonlinear, boyd2011distributed, shor2012minimization}. We can consider different optimization algorithms, e.g. the Heavy Ball method (HB) \cite{polyak1964some}, the Nesterov accelerated gradient descent (NAG) \cite{nesterov27method}, subgradient descent, Newton’s method, GD with multiple momentums \cite{pearlmutter1992gradient} and so on. Each optimization algorithm has a corresponding limiting ODE (see HB ODE \cite{polyak1964some} and NAG ODE \cite{su2014differential}), which is equivalent to the discrete optimization algorithm with an infintely small step size. To analyze such ODEs, we can apply the Gronwall's inequality \cite{gronwall1919note, bellman1943stability} for GD and the Lyapunov function or energy \cite{la2012stability,wilson2016lyapunov,polyak2017lyapunov} for higher order ODE (which is incurred by employing the momentum terms \cite{qian1999momentum, rutishauser1959theory}).

Our contribution is two-fold: we show that the non-convex weight dynamics has the same form as a strongly-convex error dynamics, where Lyapunov functions are applicable; we prove that HB also enjoys global linear convergence with a rate faster than GD, and NAG may converge at a sublinear rate yet requires more width than HB and GD.

\section{Preliminaries}
 
In this section, we introduce the NTK approach to analyze the convergence behavior of any neural networks from a dynamical system perspective. Particularly, we warm ourselves up with some known results of training a two-layer neural network \cite{du2018gradient}, using the MSE loss and the GD.

To start with, we do not specify the neural network architecture (e.g. layers, activation, depth, width, etc.). Given a training set $\{\bm{x}_i, y_i\}_{i=1}^{n}$ where $\bm{x}_i\in \mathbb{R}^{p}$, we denote the weights $\w_r\in\R^p$ as the weight vectors in the first hidden layer connecting to the $r$-th neuron, $\W$ as the union $\{\w_r\}$ and $\bm{a}$ as the set of weights in all the other layers. We write $f\left(\bm{W}, \bm{a}, \bm{x}_{i}\right)$ as the neural network output. We aim to minimize the MSE loss:
\begin{align*}
L(\bm{W}, \bm{a})=\frac{1}{2}\sum_{i=1}^{n} \left(f\left(\bm{W}, \bm{a}, \bm{x}_{i}\right)-y_{i}\right)^{2} 
\end{align*}
Taking the same route as in \cite{du2018gradient}, we focus on optimizing $\W$ with $\bm{a}$ fixed at initialization\footnote{In \Cref{sec:multi_layer} we extend our analysis to training all layers simultaneously, including the deep layers. We remark that training only the first layer is sufficient to find the global minimum of loss.}. Applying the simplest gradient descent with a step size $\eta$, we have:
\begin{align*}
\bm{w}_r(k+1)=\bm{w}_r(k)-\eta\frac{\partial L(\W(k),\bm{a})}{\partial \bm{w}_r(k)}
\end{align*}
Since GD is a discretization of its corresponding ordinary differential equation (ODE, known as the gradient flow), we analyze such ODE directly as an equivalent form of GD with an infinitesimal step size. The gradient flows are dynamical systems that are much amenable to analyze and understand different optimization algorithms. To be specific, GD has a gradient flow as 
\begin{align}
  \frac{d \bm{w}_{r}(t)}{d t}=&-\frac{\partial L(\W(t), \bm a)}{\partial \bm{w}_{r}(t)}
  \label{eq:GD flow}
\end{align}

\begin{rem}
Different optimization algorithms may have the same limiting ODE: it has been shown in \cite{hairer2006geometric} that the forward Euler discretization of \eqref{eq:GD flow} gives GD, while the backward Euler discretization gives the proximal point algorithm \cite{parikh2014proximal}. Another example relating Nesterov accelerated gradient and Heavy Ball can be found in \Cref{sec:HB}.
\end{rem}
Simple chain rules give the following dynamics,
\begin{align}
\textbf{Weight dynamics: }& \quad\quad\quad
\frac{d \bm{w}_{r}(t)}{d t}=-\frac{\partial L}{\partial \f(t)}  \frac{\partial f(t)}{\partial \bm{w}_r(t)}=-(\f-\y)\frac{\partial f(t)}{\partial \bm{w}_r(t)}
\label{eq:weight dynamics}
\\
\textbf{Prediction dynamics: }& \quad\quad\quad
\frac{d \f(t)}{d t}=  \sum_r\frac{\partial \f(t)}{\partial \bm{w}_{r}(t)}\frac{d \bm{w}_{r}(t)}{d t}= -\H(t)(\f-\y)
\label{eq:prediction dynamics}
\\
\textbf{Error dynamics: }& \quad\quad\quad
\dot{\Delta}(t) =  -\H(t)\Delta(t)
\label{eq:error dynamics}
\end{align}
Here we denote the error of the prediction $\Delta=\f-\y\in\R^n$ and the $\R^{n\times n}$ NTK matrix as 
\begin{align}
\H(t):=\sum_r\frac{\partial \f(t)}{\partial \bm{w}_{r}(t)}\left(\frac{\partial \f(t)}{\partial \bm{w}_{r}(t)}\right)^\top
\label{eq:NTK matrix}
\end{align}
which is the sum of outer products.

The key observation of training the over-parameterized neural networks is that $\W(t)$ stays very close to its initialization $\W(0)$, even though the loss may change largely. This phenomenon is well-known as `lazy training'\cite{chizat2019lazy, jacot2018neural}. As a consequence, the neural network $f$ is almost linear in $\W$ and the kernel $\H(t)$ behaves almost time-independently: $\lim_{m\to\infty}\H(0)\approx \H(0) \approx \H(t)$ \cite[Remark 3.1]{du2018gradient}.

Interestingly, suppose we define a pseudo-loss $\hat L(\Delta,\H):=\frac{1}{2}\Delta^\top \H\Delta$ and notice that $L=\frac{1}{2}\Delta^\top\Delta$, then optimizing the \textit{non-convex} loss $L$ over $\w_r$ leads to an error dynamics \eqref{eq:error dynamics}, as if we were actually optimizing a \textit{strongly-convex} loss $\hat L$ over $\Delta$ with the same dynamical system as \eqref{eq:GD flow}:
 \begin{align*}
\frac{d\Delta(t)}{dt}=-\frac{\partial \hat L(t)}{\partial \Delta(t)}.
\end{align*}
The matrix ODE \eqref{eq:error dynamics} with a constant and positive $\H$ has a solution converging to 0 at linear rate, as the classical theory on optimizing a strongly convex loss indicates. In other words, optimizing the non-convex $L$ for over-parameterized neural networks converges faster than optimizing convex losses and reaches the convergence speed of optimizing strongly convex losses.  Therefore it is essential to show that $\H$ is positive with the smallest eigenvalue bounded away from 0 at all time. We formalize this claim by quoting the results for the two-layer neural network of the following form,
\begin{align*}
f(\bm{W}, \bm{a}, \bm{x})=\frac{1}{\sqrt{m}} \sum_{r=1}^{m} {a}_{r} \sigma\left(\bm{w}_{r}^{\top} \bm{x}\right) 
\end{align*}
with $\sigma(z)=\max\{z,0\}$ being the ReLU activation function.
Now we quote an important fact that justifies our main theorem.

\begin{fact}[Assumption 3.1 and Theorem 3.1 in \cite{du2018gradient}]
\label{fact:gram}
Define matrix $\H^{\infty} \in\mathbb{R}^{n \times n}$ with 
\begin{align*}
(\H^{\infty})_{i j}=& \mathbb{E}_{\w_r \sim N(\mathbf{0}, \mathbf{I})}\left[\x_{i}^{\top} \x_{j} \mathbb{I}\left\{\w_r^{\top} \x_{i} \geq 0, \w_r^{\top} \x_{j} \geq 0\right\}\right]
= \left(\frac{1}{2}-\frac{\arccos(\bm x_i^\top\bm x_j)}{2 \pi}\right)(\bm x_i^\top \bm x_j)
\end{align*}
and define $\lambda_{0}:=\lambda_{\min }\left(\H^{\infty}\right)$ as the smallest eigenvalue of $\H^\infty$. Suppose for any $i\neq j$, $\x_{i} \nparallel \x_{j}, \text { then } \lambda_{0}>0$.
\end{fact}
Here $\H^\infty$ is the limiting form at the initialization, i.e, $\H^{\infty} = \lim_{m \to \infty} \H(0)$. We note that the explicit formula $(\H^{\infty})_{i j}= \left(\frac{1}{2}-\frac{\arccos(\bm x_i^\top\bm x_j)}{2 \pi}\right)(\bm x_i^\top \bm x_j)$ is given by \cite{cho2009kernel} and especially $(\H^{\infty})_{ii}=1/2$. In \cite{du2018gradient}, the authors establish that, for sufficiently wide hidden layer and under some data distributional assumptions, GD converges to zero training loss exponentially fast. Formally, 
\begin{theorem}[Theorem 3.2 and Lemma 3.2 in \cite{du2018gradient}]
\label{thm:simon}
Suppose $\forall i, \|\x_i\|_2=1$ and $|y_i|<C$ for some constant $C$, and only the hidden layer weights $\{\w_r\}$ are optimized by GD. If we set the width $m=\Omega(n^6/\lambda_0^4\delta^3)$ and we i.i.d. initialize $\bm{w}_{r} \sim \mathcal{N}(\mathbf{0}, \mathbf{I}), a_{r} \sim \operatorname{unif}\{-1,1\} $ for $ r \in[m], $ then with high
probability at least $ 1-\delta $ over the initialization, we have
$$\lambda_{\min}(\H(t))>\frac{1}{2}\lambda_{\min}(\lim_m\H(0)):=\frac{1}{2}\lambda_{\min}(\H^\infty)=\frac{\lambda_0}{2}$$
with $\H^\infty$ defined in \Cref{fact:gram}, and we have the linear convergence
\begin{align*}
L(t) \leq \exp \left(-\lambda_0 t\right)L(0).
\end{align*}
\end{theorem}
We note that the NTK matrix \eqref{eq:NTK matrix} is the Gram matrix induced by the ReLU activation: 
\begin{align}
\H_{ij}(t)
=\sum_{r=1}^m \frac{\partial f_i(t)}{\partial \bm w_r}\left(\frac{\partial f_j(t)}{\partial \bm w_r}\right)^\top
=\frac{1}{m} \bm x_i^\top \bm x_j \sum_{r=1}^m \mathbb{I}(\bm w_r^\top \bm x_i\geq 0,\bm w_r^\top \bm x_j\geq 0)
\label{eq:H(t) definition}
\end{align}
We remark that the framework of \Cref{thm:simon} has been extended to training multiple layers simultaneously \cite{du2018gradient} (see also \Cref{sec:multi_layer}). The analysis has recently been generalized to different network architectures and losses and we now complement this line of research by extending to different optimization algorithms. To present the simplest proof, we only analyze the continuous gradient flows and we believe our approach can be easily extended to the discrete time analysis.

\section{Heavy Ball with Friction System}
\label{sec:HB}
 

Our first result concerns the GD with momentum, or the Heavy Ball (HB) method \cite{polyak1964some}:
\begin{equation}
    \begin{split}
        \bm w_r(k+1)=\bm w_r(k)-\eta\frac{\partial L(\W(k),\bm a)}{\partial \bm w_r(k)} +\beta(\bm w_r(k)-\bm w_r(k-1))
\label{eq:HB discrete}
\end{split}
\end{equation}
or equivalently
\begin{equation}
\begin{split}
\bm w_r(k+1) =& \bm w_r(k) + \eta \bm v(k)\\
\bm v(k)=& -\frac{\partial L(\W(k),\bm a)}{\partial \bm w_r(k)} + \beta \bm v(k-1)
\end{split}
\end{equation}
where $\eta$ is the step size and the momentum term $\beta\in[0,1]$. The corresponding gradient flow is known as the Heavy Ball with Friction (HBF) system. This is a non-linear dissipative dynamical system, originally proposed by \cite{polyak1964some} and heavily studied
in \cite{attouch2000heavy,gadat2018stochastic,cabot2009long, attouch2000heavy2,alvarez2002second,wilson2016lyapunov,loizou2020momentum,liu2020improved}: with $b>0$\footnote{Here we can see $b\sim\frac{1-\beta}{\sqrt{\eta}}$ by discretizing the HBF system.}
\begin{align}
\ddot{\bm w_r}(t)+b\dot{\bm w_r}(t)+\frac{\partial L(\W(t),\bm a)}{\partial \bm w_r(t)}=0.
\label{eq:heavy ball flow}
\end{align}
As shown later in \Cref{proof:HBF thm}, the error dynamics is
\begin{align}
\ddot{\Delta}(t)+b\dot{\Delta}(t)+\frac{\partial \hat L}{\partial \Delta(t)}\overset{\text{a.s.}}{=}0.
\label{eq:heavy ball error flow}
\end{align}
We note that other optimization algorithms may also correspond to the HBF system: for example NAG-SC (Nesterov accelerated gradient descent for strongly convex objective \cite{wilson2016lyapunov}, also see \Cref{sec:NAGSC}), though NAG-SC and HB are distinguishable using high resolution ODE \cite{shi2018understanding}.

In particular, we study the case as in \cite[Equation (7)]{wilson2016lyapunov} and \cite{siegel2019accelerated}, when $b=\sqrt{2\lambda_0}$, i.e. twice the strongly convexity of $\H(t)$:
\begin{align}
\ddot{\bm w_r}(t)+\sqrt{2\lambda_0}\dot{\bm w_r}(t)+\frac{\partial \f}{\partial \bm w_r}(\f-\y)=0
\label{eq:HB special b}
\end{align}
Our choice of parameter $b$ leads to a global linear convergence to zero training loss, without requiring Lipschitz gradients of $\hat L$. For other choices of parameters with the Lipschitz condition of $\hat L$, HB can enjoy linear converge locally \cite{polyak1964some,lessard2016analysis} and globally \cite{nesterov27method,ghadimi2015global,van2017fastest,siegel2019accelerated,aujol2020convergence}. 

To solve a second order ODE requires initial conditions on $\w_r$ and $\dot\w_r$, which we assume as $\dot\w_r(0)=0$ without loss of generality. Now we state the our main theorem under MSE loss.
\begin{theorem}\label{thm:HBF linear rate}
Suppose we set the width of the hidden layer $m=\Omega\left(\frac{n^6}{\delta^3\lambda_0^4}\right)$, $b=\sqrt{2\lambda_0}$ and we train with HB. If we i.i.d. initialize $\bm{w}_{r} \sim \mathcal{N}(\mathbf{0}, \mathbf{I}), a_{r} \sim \operatorname{unif}\{-1,1\} $ for $ r \in[m]$, then with high
probability at least $ 1-\delta $ over the initialization, we have
\begin{align*}
 L(t) &\leq\frac{4}{\lambda_0}\exp \left(-\sqrt{\lambda_0/2}\cdot t\right)\hat L(0)
\end{align*}
\end{theorem}

We notice that indeed $\sqrt{\lambda_0/2}>\lambda_0$, suggesting a boost in the linear convergence rate of HB when compared to GD, as we observe in \Cref{fig:HB}. To prove this, we claim that $\lambda_0<1/2$ as tr$(\H^\infty)=n\cdot(\H^{\infty})_{i i}=n/2=\sum_i \lambda_i>n\lambda_0$, with $\lambda_i$ representing the eigenvalues of $\H^\infty$. Another observation is that, to guarantee the linear convergence, HB requires the same order of width as GD in \Cref{thm:simon}. In \Cref{proof:HBF thm}, we prove our theorem by employing the Lyapunov function.

\begin{figure}[!htb]
\centering
\includegraphics[width=8cm]{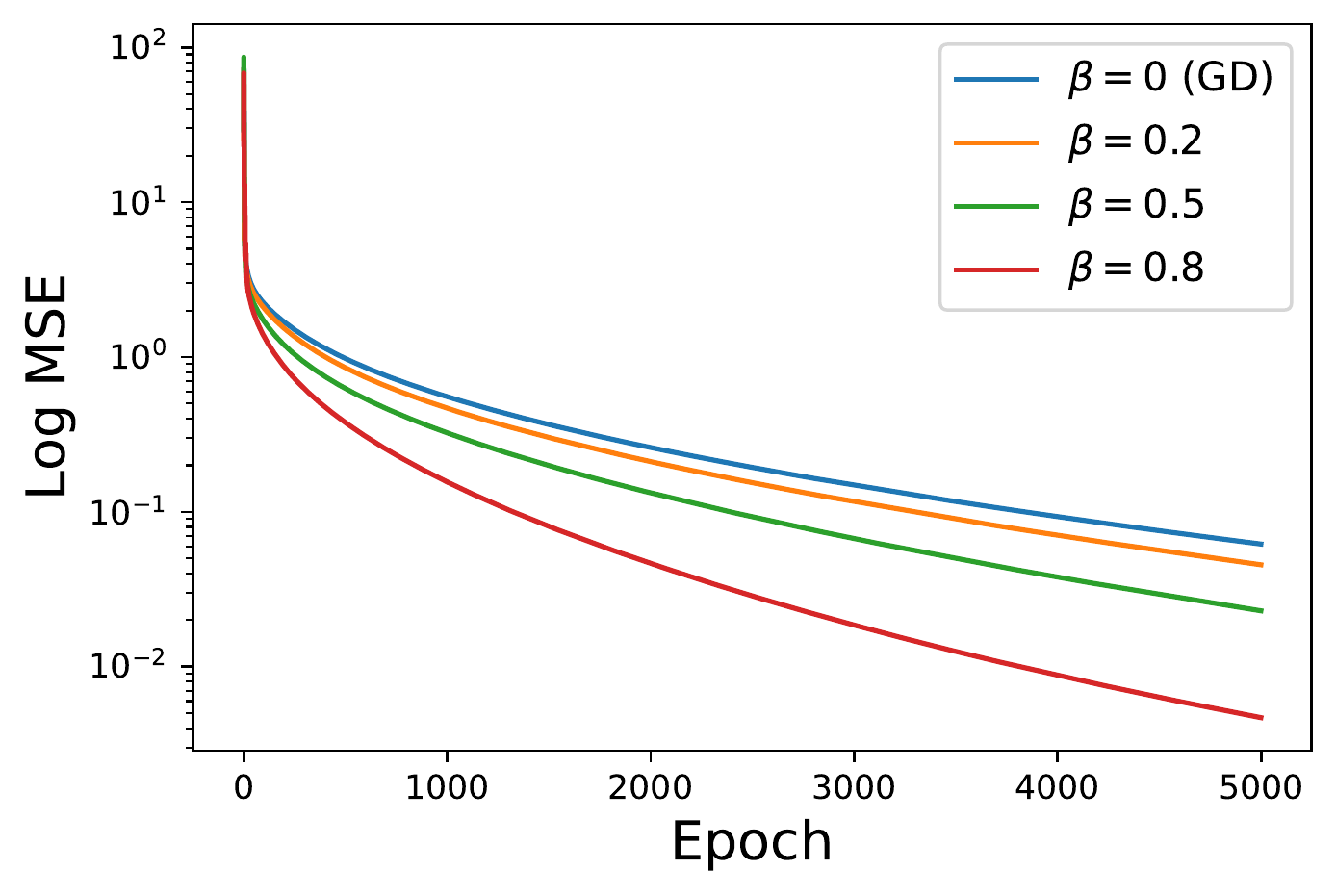}
\caption{Logarithm MSE vs. epoch of HB on MNIST. We observe that HB shows faster linear convergence than GD. Details of experiments are left in \Cref{exp:hb}.}
\label{fig:HB}
\end{figure}

\subsection{Proof of \Cref{thm:HBF linear rate}}\label{proof:HBF thm}
First, we use the chain rule to characterize the prediction dynamics of $f$,
\begin{align*}
\dot{f}_i(t)&=\sum_{r\in[m]}\frac{\partial f_i}{\partial \w_r}\dot{\w_r}
\\
\ddot{f}_i(t)&=\sum_{r,l\in[m]}\dot{\w_r}^\top\frac{\partial^2 f_i}{\partial \w_r\partial \w_l}\dot{\w_l}+\sum_{r\in[m]}\frac{\partial f_i}{\partial \w_r}\ddot{\w_r}\overset{\text{a.s.}}{=}\sum_{r\in[m]}\frac{\partial f_i}{\partial \w_r}\ddot{\w_r}
\end{align*}
where the last equality follows from a key observation that, with $\delta(\cdot)$ denoting the Dirac Delta function,
\begin{align}
\begin{split}
\frac{\partial f_i}{\partial \bm w_r}&=\frac{1}{\sqrt{m}}a_r\bm x_i\mathbb{I}(\bm w_r^\top \bm x_i>0),
\\
\frac{\partial^2 f_i}{\partial \bm w_r^2}&=\frac{1}{\sqrt{m}}a_r\bm x_i^\top\bm x_i\delta(\bm w_r^\top \bm x_i)\overset{\text{a.s.}}{=} 0,
\\
\frac{\partial^2 f_i}{\partial \bm w_r\partial \bm w_l}&=\mathbf{0} \text{\quad for $l\neq r$}.
\end{split}
\label{eq:second derivative zero}
\end{align}
Multiplying $\frac{\partial \f}{\partial \bm{w}_r}$ to \eqref{eq:HB special b} and sum over $r$, we obtain the prediction dynamics as
\begin{align*}
\ddot{\f}(t)+\sqrt{2\lambda_0}\dot{\f}(t)+\H(t)(\f-\bm{y})\overset{\text{a.s.}}{=} 0
\end{align*}
and consequently, the dynamics of the error is
\begin{align}
\ddot{\Delta}(t)+\sqrt{2\lambda_0}\dot{\Delta}(t)+\H(t)\Delta(t)\overset{\text{a.s.}}{=}0
\label{eq:deltadyn}
\end{align}
or in an analogous form to \eqref{eq:HB special b},
\begin{align*}
\ddot{\Delta}(t)+\sqrt{2\lambda_0}\dot{\Delta}(t)+\frac{\partial \hat L}{\partial \Delta(t)}\overset{\text{a.s.}}{=} 0
\end{align*}
In what follows, we drop the \textit{a.s.} and focus on the ordinary differential equations. To establish the linear convergence in MSE, we need to guarantee $\H(t)$ is positive definite with $\lambda_{\min}(\H(t))\geq\lambda_0/2$ for all $t$. In other words, the pseudo loss $\hat L$ is $\frac{\lambda_0}{2}$-strongly convex. We start with $t=0$, by showing that for wide enough neural networks, $\H(0)$ has a positive smallest eigenvalue with high probability. \begin{lemma}[Lemma 3.1 in \cite{du2018gradient}]
If $m=\Omega\left(\frac{n^2}{\lambda_{0}^{2}} \log \left(\frac{n^2}{\delta}\right)\right)$, then we have $\left\|\H(0)-\H^{\infty}\right\|_{2} \leq \frac{\lambda_{0}}{4}$ and $\lambda_{\min }(\H(0)) \geq \frac{3}{4} \lambda_{0}$ with probability of at least $1-\delta$.
\label{lem:B1}
\end{lemma}
Next, we introduce a lemma that shows for any $t$, if $\w_r(t)$ is close to $\w_r(0)$, then $\H(t)$ is close to $\H(0)$. Together with \Cref{lem:B1}, $\lambda_{\min}(\H(t))$ always has a positive smallest eigenvalue. In words, the lazy training leads to the positive definiteness.
\begin{lemma}[Lemma 3.2 in \cite{du2018gradient}]
If $\bm{w}_r$ are i.i.d. generated from $\mathcal{N}(\bm{0},\bm{I})$ for $r \in [m]$, and $\left\|\bm{w}_{r}(0)-\bm{w}_{r}\right\|_{2} \leq \frac{c \delta \lambda_{0}}{n^2} =: R$ for some small positive constant $c$, then the following holds with probability at least $1-\delta$: we have $\|\bm{H}(t)-\bm{H}(0)\|_2<\frac{\lambda_0}{4}$ and $\lambda_{\min}(\bm{H}(t))>\frac{\lambda_0}{2}$.
\label{lem:B2}
\end{lemma}

The next lemma gives two important facts given that $\lambda_{\min}(\H(s))$ for previous time $s\leq t$: the loss decays exponentially and weights stay close to their initialization at the current time $t$. In other words, the positive definiteness indicates the convergence of the loss, which further indicates the lazy training. We emphasize that our \Cref{lem:B3} is specific to the choice of optimization algorithms and hence the proof is much different than its analog in \cite[Lemma 3.3]{du2018gradient} for GD.

\begin{lemma}
Assume $0\leq s \leq t$ and $\lambda_{\min}(\bm{H}(s))\geq \frac{\lambda_0}{2}$. Then $L(t) \leq \frac{4}{\lambda_0}\exp \left(-\sqrt{\lambda_0/2}t\right) \hat L(0)$ and $\left\|\bm{w}_{r}(t)-\bm{w}_{r}(0)\right\|_{2} \leq \sqrt{\frac{256 n\hat L(0)}{9m\lambda_0^3}} =: R^{\prime}$.
\label{lem:B3}
\end{lemma}
\begin{proof}[Proof of \Cref{lem:B3}]
Borrowing the idea of \cite{siegel2019accelerated}, we define the Lyapunov function or Lyapunov energy as
\begin{align*}
V(t):=&\hat L+\frac{1}{2}\left\|\sqrt{\frac{\lambda_0}{2}}\Delta(t)+\dot\Delta(t)\right\|^2
=\frac{1}{2}\Delta(t)^\top\H(t)\Delta(t)+\frac{1}{2}\left\|\sqrt{\frac{\lambda_0}{2}}\Delta(t)+\dot\Delta(t)\right\|^2.
\end{align*}
The Lyapunov function represents the total energy of the system and always decreases along the trajectory of the training dynamics since, as we will later show, $\dot V(t)< 0$. Here we simplify the notation by denoting the dependence on $t$ in the subscript and use $\alpha:=b/2=\sqrt{\lambda_0/2}$. We derive by the chain rule,
\begin{align*}
\dot V(t)&=\dot{\Delta}_{t}^\top\H(t)\Delta_t+\frac{1}{2}\Delta_t^\top\dot\H(t)\Delta_t+\left\langle{\alpha} \dot{\Delta}_{t} +\ddot{\Delta}_{t}, {\alpha}\Delta_{t}+\dot\Delta_{t}\right\rangle,
\end{align*}
where we use $\langle\bm u,\bm v\rangle=\bm u^\top\bm v$ to denote the inner product.

Notice that by \eqref{eq:H(t) definition}, we have $\dot\H(t)\overset{\text{a.s.}}{=}0$. Substituting the error dynamics \eqref{eq:deltadyn} for $\ddot\Delta_t$, we have
\begin{align*}
\dot V(t)&=\left\langle\H\Delta_t, \dot\Delta_{t}\right\rangle+\left\langle-{\alpha} \dot\Delta_{t}-\H\Delta_{t}, {\alpha}\Delta_{t}+\dot\Delta_t\right\rangle
\\
&=-{\alpha}\left\langle\H\Delta_{t},\Delta_{t}\right\rangle-\alpha^2\left\langle \dot\Delta_{t},\Delta_t\right\rangle-{\alpha}\left\langle \dot\Delta_{t}, \dot\Delta_{t}\right\rangle
\end{align*}
Using $\lambda_{\min}(\H)\geq\frac{\lambda_0}{2}=\alpha^2$, we get
\begin{align*}
\left\langle\H\Delta_{t},\Delta_{t}\right\rangle \geq\frac{1}{2}\left\langle\H\Delta_{t},\Delta_{t}\right\rangle+\frac{\alpha^2}{2}\langle\Delta_{t},\Delta_{t}\rangle
=\hat L(t)+\frac{\alpha^2}{2}\langle\Delta_{t},\Delta_{t}\rangle
\end{align*}
and hence we have
\begin{align*}
\dot V(t)=&-{\alpha}\hat L(t)-{\frac{\alpha^3}{2}}\langle\Delta_{t},\Delta_{t}\rangle-\alpha^2\left\langle \dot\Delta_{t},\Delta_t\right\rangle-{\alpha}\left\langle \dot\Delta_{t}, \dot\Delta_{t}\right\rangle
\\
<&-{\alpha}\left(\hat L(t)+\frac{1}{2}\left\|{\alpha}\Delta_{t}+\dot\Delta_{t}\right\|^2\right)=-{\alpha}V(t)
\end{align*}
where in the last inequality we throw away $-\frac{{\alpha}}{2}\left\langle \dot\Delta_{t}, \dot\Delta_{t}\right\rangle$. Clearly $\dot V(t)< 0$ for all $t$ since $V(t)$ is by definition positive. For this first order scalar ODE, we apply the Gronwall's inequality to derive that
\begin{align*}
V(t)<e^{-{\alpha}t}V(0)
\end{align*}
and we obtain
\begin{align*}
\hat L(t) &\leq V(t)<e^{-{\alpha}t}V(0)
=e^{-{\alpha}t}\left(\frac{1}{2}\Delta(0)^\top\H(0)\Delta(0)+\frac{\alpha^2}{2}\|\Delta(0)\|^2\right).
\end{align*}
Again using $\lambda_{\min}(\H(0))\geq\alpha^2$, we have
\begin{align*}
 \hat L(t) &\leq  \exp \left(-{\alpha} t\right)\left( \frac{1}{2} \Delta(0)^\top\H(0)\Delta(0) +  \hat{L}(0) \right)
 = 2\exp \left(-{\alpha} t\right)\hat L(0)
\end{align*}
and 
$$L(t)\leq\frac{2}{\alpha^2}\exp \left(-\alpha t\right)\hat L(0).$$

In words, the prediction $ \f(t) \rightarrow \mathbf{y} $ exponentially fast, with a convergence factor $\alpha=\sqrt{\lambda_0/2}$. 

Now we move on to show that $\w_r(t)$ stays close to $\w_r(0)$. Multiplying $e^{bt}=e^{2\alpha t}$ to the weight dynamics \eqref{eq:HB special b}, we have
\begin{align*}
\frac{d}{dt}\left(e^{2\alpha t}\dot \w_r\right)=-\frac{1}{\sqrt{m}}e^{2\alpha t}a_r\sum_i(f_i-y_i)\x_i \I(\w_r^\top\x_i\geq 0)
\end{align*}
which gives a close-form solution
\begin{align*}
\dot \w_r=-e^{-2\alpha t}\int_0^t\frac{1}{\sqrt{m}}e^{2\alpha s}a_r\sum_i(f_i-y_i)\x_i \I(\w_r^\top\x_i\geq 0)ds
\end{align*}
whose norm satisfies
\begin{align*}
\|\dot \w_r(t)\|
&\leq e^{-2\alpha t}\frac{1}{\sqrt{m}}\int_0^t e^{2\alpha s}\sum_i|f_i(s)-y_i|ds\\
&\leq e^{-2\alpha t}\sqrt{\frac{n}{m}}\int_0^t e^{2\alpha s}\|\f(s)-\y\|_2 ds\\
&=\sqrt{\frac{n}{m}}\int_0^t e^{2\alpha (s-t)}\sqrt{L(s)}ds\\
&\leq \sqrt{\frac{2n\hat L(0)}{m\alpha^2}}\int_0^t e^{\frac{3}{2}\alpha s-2\alpha t}ds\\
&\leq \sqrt{\frac{8n\hat L(0)}{9m\alpha^4}}e^{-\frac{\alpha}{2} t}
\end{align*}
Finally by Cauchy Schwarz,
we bound the weight distance from initialization,
\begin{align*}
 &\left\|\w_{r}(t)-\w_{r}(0)\right\|_{2}
 \leq \int_{0}^{t}\left\|\dot\w_{r}(s)\right\|_{2} d s < \sqrt{\frac{32n\hat L(0)}{9m\alpha^6}}=\sqrt{\frac{256 n\hat L(0)}{9m\lambda_0^3}}
\end{align*}
\end{proof}

We quote the next lemma to show that $R'<R$ indicates that for all $t>0$, the conditions in \Cref{lem:B2} and \Cref{lem:B3} hold. This lemma is closely related to \cite[Lemma 3.4]{du2018gradient} and the proof is given in \Cref{3442}.
\begin{lemma}
If $R'< R$, then we have $\lambda_{\min}(\bm{H}(t))\geq \frac{\lambda_0}{2}$, for all $r\in [m]$, $\|\bm{w}_r(t)-\bm{w}_r(0)\|_2 \leq R'$ and  $L(t) \leq \frac{4}{\lambda_0}\exp \left(-\sqrt{\lambda_0/2}t\right)\hat L(0)$.
\label{lem:B4}
\end{lemma}
Finally we study the width requirement for $R'<R$ to hold true, i.e. we need
$\sqrt{\frac{256n\hat L(0)}{9m\lambda_0^3}}<O(\frac{\delta\lambda_0}{n^2})$ which is equivalent to $m=\Omega(\frac{n^6}{\delta^3\lambda_0^4})$. This is shown in \Cref{width HB}.

\section{Nesterov Accelerated Gradient Method}

In this section, we analyze the dynamics of the generalized Nesterov Accelerated Gradient (NAG) descent as follows,
\begin{align}
\bm v(k+1)&=\w_r(k)-\eta\frac{\partial L(\W(k),\bm a)}{\partial \bm w_r(k)}
\label{eq:NAG discrete}
\\
\w_r(k+1)&=\bm v(k+1)+\frac{k-1}{k+\gamma-1}(\bm v(k+1)-\bm v(k))
\nonumber
\end{align}
where $\eta$ is step size. We remark that the NAG \eqref{eq:NAG discrete} has a time-dependent momentum coefficient $\frac{k-1}{k+\gamma-1}$ while in practice, e.g. in \cite{sutskever2013importance} (adopted in Pytorch and Tensorflow), a time-independent momentum is commonly used. In fact, a special case of NAG with time-independent momentum is analyzed in \Cref{sec:NAGSC}. It has been given in \cite{su2014differential} that the corresponding gradient flow as
\begin{align}
\ddot{\bm w_r}(t)+\frac{\gamma}{t}\dot{\bm w_r}(t)+\frac{\partial L(\W(t),\bm a)}{\partial \bm w_r(t)}=0.
\label{eq:NAG flow}
\end{align}
With initial conditions $\dot{\bm w_r}(0)=0$, we multiply $\frac{\partial \f}{\partial \bm{w}_r}$ to (\ref{eq:NAG flow}) and sum over $r$. It follows from \eqref{eq:second derivative zero} that the prediction dynamics and the error dynamics are
\begin{align}
\ddot{\f}(t)+\frac{\gamma}{t}\dot{\f}(t)+\H(t)(\f-\y)&\overset{\text{a.s.}}{=}0
\\
\ddot{\Delta}(t)+\frac{\gamma}{t}\dot{\Delta}(t)+\H(t)\Delta(t)&\overset{\text{a.s.}}{=}0
\end{align}
with the same NTK matrix $\H(t)$ as defined in (\ref{eq:H(t) definition}). Again, using the pseudo-loss $\hat L(t)$, we have an error dynamics that is analogous to the weight dynamics \eqref{eq:NAG flow},
\begin{align}
\ddot{\Delta}(t)+\frac{\gamma}{t}\dot{\Delta}(t)+\frac{\partial \hat L}{\partial\Delta(t)}&\overset{\text{a.s.}}{=}0.
\label{eq:NAG error dynamics}
\end{align}
We now state our convergence analysis of NAG based on this error dynamics.
\begin{theorem}\label{thm:Nesterov rate}
Suppose we set the width of the hidden layer $m=\Omega\left(\frac{n^{5\alpha/2 - 4}}{\delta^{3\alpha/2 - 3} \lambda_0^{3\alpha/2 - 2}}\right)$ where $4 < \alpha \leq \frac{2\gamma}{3}$ and $\gamma> 6$, and we train with NAG. If we i.i.d. initialize $\bm{w}_{r} \sim \mathcal{N}(\mathbf{0}, \mathbf{I}), a_{r} \sim \operatorname{unif}\{-1,1\} $ for $ r \in[m]$, then with high
probability at least $ 1-\delta $ over the initialization, we have
\begin{align}
 L(t) \leq  A(\alpha,\gamma,\lambda_0)t^{-\alpha}   {L}(0) 
\end{align}
where $A(\alpha,\gamma,\lambda_0)$ is a constant that only depends on $\alpha$, $\gamma$ and $\lambda_0$, defined in \eqref{eq:A definition}. 
\end{theorem}

We pause here to discuss the choice of $\gamma$ in \eqref{eq:NAG discrete}. In \cite{su2014differential}, the `magic constant' $\gamma$ has been extensively studied. When $\gamma\geq 3$, the convergence rate is shown to be $O(t^{-\frac{2\gamma}{3}})$. When $\gamma<3$, there exist counter-examples that fail the desired $O(1/t^2)$ convergence rate. We remark that we only need $\gamma>3$ to derive the convergence but we assume $\gamma>6$ only to guarantee the lazy training $\w_r(t)\approx \w_r(0)$. 

From \Cref{thm:Nesterov rate}, NAG only converges at polynomial rate, in contrast to the linear convergence of HB and GD. This can be visualized as the linear pattern in \Cref{fig:NAG} (note the GD pattern is concave). Therefore, in the long run when $t$ is sufficiently large, NAG may be outperformed by GD. In addition, NAG and HB are well-known to have non-monotone loss dynamics, which is different than GD.

\begin{figure}[!htb]
\centering
\includegraphics[width=8cm]{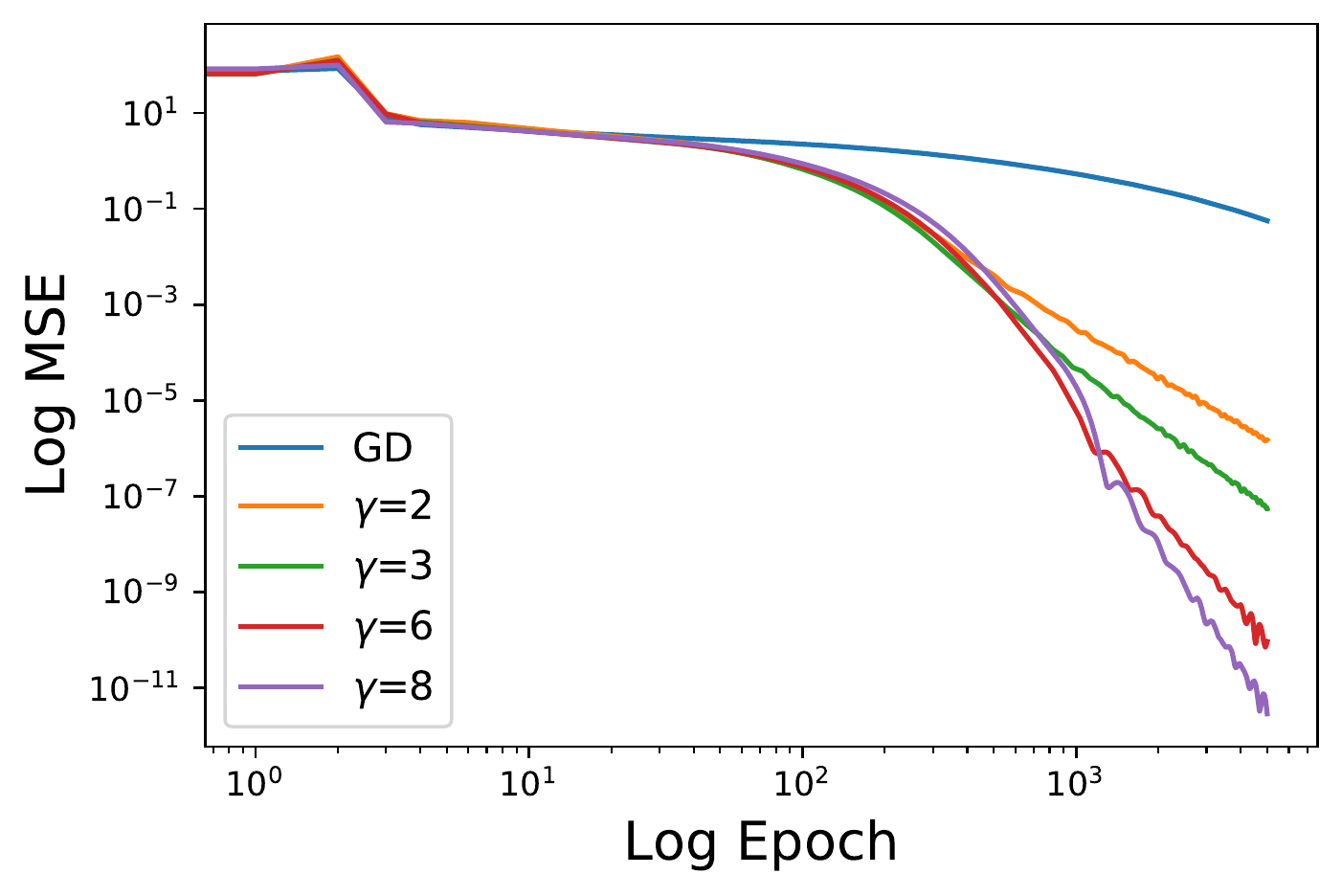}
\caption{Logarithm MSE vs. logarithm epoch of NAG on MNIST. We observe that NAG shows polynomial decay in loss but converges faster than GD with learning rate $5\times 10^{-5}$. Details of experiments are left in \Cref{exp:nag}.}
\label{fig:NAG}
\end{figure}

\subsection{Proof of \Cref{thm:Nesterov rate}}
To prove \Cref{thm:Nesterov rate}, we use the same framework as in  \Cref{sec:HB}: given \Cref{lem:B1} and \Cref{lem:B2}, we will prove \Cref{lem:B3+} as an analogy to \Cref{lem:B3}, but customized for NAG.



\begin{lemma} \label{lem:B3+}
Assume $0 \leq s \leq t$ and $\lambda_{\min} (\bm{H}(s) ) \geq \frac{\lambda_0}{2}$. Then we have $L(t) \leq \frac{C(\alpha,\gamma)}{t^\alpha (\lambda_0/2)^{\frac{\alpha}{2}}}  {L}(0)$ for $2 \leq \alpha \leq \frac{2\gamma}{3}$ and some $C(\alpha,\gamma)$ only depending on $\alpha$ and $\gamma$. Furthermore, if $\alpha>4$, we have 
\begin{align*} 
\|\w_{r}(t)-\w_{r}(0)\|_{2}\leq R'(\epsilon) :=\frac{ 2\epsilon^{2 - \alpha/2} }{\alpha - 4} \sqrt{\frac{nC(\alpha,\gamma){L}(0)}{m  (\lambda_0/2)^{\frac{\alpha}{2}}(\gamma - \alpha/2 + 1)^2}  }  + \sqrt{2L(0)} \epsilon.
\end{align*}
\end{lemma}

\begin{proof}[Proof of \Cref{lem:B3+}]
We define the Lyapunov function as in \cite{su2014differential}
\begin{align*}
V(t; \alpha, \gamma):=t^{\alpha}\hat{L}(t)+\frac{(2 \gamma-\alpha)^{2} t^{\alpha-2}}{8}\left\|\Delta_t+\frac{2 t}{2 \gamma-\alpha} \dot{\Delta}_t\right\|^{2}
\end{align*}
and use the same loss $L$ and pseudo-loss $\hat L$ as before.

For $\gamma > 3$ and $2 \leq \alpha \leq \frac{2\gamma}{3}$, we obtain from \cite[Theorem 8]{su2014differential} (restated in \Cref{app:su nesterov}), with some $C(\alpha,\gamma)$ only depending on $\alpha$ and $\gamma$, the error dynamics \eqref{eq:NAG error dynamics} leads to
\begin{align*}
    \hat L(t) \leq \frac{C(\alpha,\gamma)}{t^\alpha (\lambda_0/2)^{\frac{\alpha}{2}-1}}L(0):=\frac{A(\alpha,\gamma,\lambda_0)}{t^\alpha (\lambda_0/2)^{-1}}L(0).
\end{align*}
Leveraging $\lambda_{\min}(\H)\geq\frac{\lambda_0}{2}$, we have $\hat L(t)\geq\frac{\lambda_0}{2}L(t)$ and
\begin{align}
    L(t) \leq\frac{2}{\lambda_0}\hat L(t)
    \leq \frac{C(\alpha,\gamma)}{ t^\alpha (\lambda_0/2)^{\frac{\alpha}{2}}}L(0).
\label{eq:A definition}
\end{align}
We denote $A(\alpha,\gamma,\lambda_0):=\frac{C(\alpha,\gamma)}{ (\lambda_0/2)^{\frac{\alpha}{2}}}$ for brevity of statements in \Cref{thm:Nesterov rate}. 

As for the lazy training, since the weight dynamics is
\begin{align*}
    \ddot{\bm w_r}(t)+\frac{\gamma}{t}\dot{\bm w_r}(t) + \frac{\partial \f}{\partial \bm w_r}(\f - \y) = 0,
\end{align*}
multiplying both sides with $t^\gamma$, we obtain
\begin{align*}
    \frac{d}{dt} (t^\gamma \dot{\bm w}_r(t)) = - t^\gamma \frac{\partial \f}{\partial \bm w_r}(\f - \y)    =-\frac{1}{\sqrt{m}}t^\gamma a_r\sum_i(f_i-y_i)\x_i \I(\w_r^\top\x_i\geq 0) 
\end{align*}
which gives a closed-form solution
\begin{align*}
\dot \w_r=- \frac{1}{t^\gamma}\int_0^t\frac{1}{\sqrt{m}}s^\gamma a_r\sum_i(f_i-y_i)\x_i \I(\w_r^\top\x_i\geq 0)ds
\end{align*}
whose norm satisfies
\begin{align*}
\|\dot \w_r(t)\|_2
&\leq \frac{1}{t^\gamma\sqrt{m}}\int_0^t s^\gamma\sum_i|f_i(s)-y_i|ds
\\
&\leq \frac{1}{t^\gamma} \sqrt{\frac{n}{m}}\int_0^t s^\gamma\|\f(s)-\y\|_2 ds
\\
&=\frac{1}{t^\gamma}\sqrt{\frac{n}{m}}\int_0^t s^\gamma\sqrt{L(s)}ds
\\
&\leq \frac{1}{t^\gamma} \sqrt{\frac{nC(\alpha,\gamma)L(0)}{m   (\lambda_0/2)^{\frac{\alpha}{2}}}  } \int_0^t s^{\gamma - \alpha/2} ds \\
&=t^{1 - \alpha/2  } \sqrt{\frac{nC(\alpha,\gamma)L(0)}{m  (\lambda_0/2)^{\frac{\alpha}{2}}(\gamma - \alpha/2 + 1)^2}  }.  
\end{align*}
Note $t^{1-\alpha/2}$ is not defined at $t=0$. Next we bound the weight distance by breaking the integral into two pieces, for any $0 < \epsilon < t$:
\begin{align*}
 \left\|\w_{r}(t)-\w_{r}(0)\right\|_{2} 
 &\leq \int_{0}^{t}\left\|\dot\w_{r}(s)\right\|_{2} d s
 = \int_{\epsilon}^{t}\left\|\dot\w_{r}(s)\right\|_{2} d s + \int_{0}^{\epsilon}\left\|\dot\w_{r}(s)\right\|_{2} d s \\
 &\leq \frac{ 2\epsilon^{2 - \alpha/2} }{\alpha - 4} \sqrt{\frac{nC(\alpha,\gamma){L}(0)}{m  (\lambda_0/2)^{\frac{\alpha}{2}}(\gamma - \alpha/2 + 1)^2}  }+ \int_{0}^{\epsilon}\left\|\dot\w_{r}(s)\right\|_{2} d s.
\end{align*}
Now we need to bound $\left\|\dot\w_{r}(s)\right\|_{2}$ in $\int_{0}^{\epsilon}\left\|\dot\w_{r}(s)\right\|_{2} d s$ with a time-independent upper bound, different than the previous $O(t^{1-\alpha/2})$ one. We achieve this goal by analyzing another Lyapunov function from \cite{polyak2017lyapunov}
\begin{align*}
    \mathcal{E}(t) =   L(\W(t),\bm a) + \frac{1}{2} \sum_{r} \dot{\bm w_r}(t)^\top \dot{\bm w_r}(t).
\end{align*}
By simple differentiation, we see that $\mathcal{E}$ is decreasing in $t$ since $\dot{ \mathcal{E}}(t)=- \sum_r \frac{\gamma}{t} \|\dot{\bm w_r}(t) \|^2_2\leq 0$, which is obvious using $\frac{d L}{dt}=\sum_r\frac{\partial L}{\partial \w_r}\dot\w_r$. 
This implies that
\begin{align*}
 \frac{1}{2}  \sum_{r}  \dot{\bm w_r}(t)^\top \dot{\bm w_r}(t) \leq \mathcal{E} (t) \leq  \mathcal{E} (0) = L(0). 
\end{align*}
Hence we obtain from $\|\dot\w_r(s)\|^2\leq\sum_r\|\dot\w_r(s)\|^2$ that
\begin{align*}
\int_{0}^{\epsilon}\left\|\dot\w_{r}(s)\right\|_{2} d s 
\leq \int_{0}^{\epsilon}\sqrt{\sum_r\left\|\dot\w_{r}(s)\right\|^2} d s 
\leq \sqrt{2L(0)} \epsilon.
\end{align*}
Therefore, for any $\epsilon$ and $m$, we have
\begin{align*}
\|\w_{r}(t)-\w_{r}(0)\|_{2}
  \leq R'(\epsilon)
  :=\frac{ 2\epsilon^{2 - \alpha/2} }{\alpha - 4} \sqrt{\frac{nC(\alpha,\gamma){L}(0)}{m  (\lambda_0/2)^{\frac{\alpha}{2}}(\gamma - \alpha/2 + 1)^2}  }  + \sqrt{2L(0)} \epsilon\end{align*}
\end{proof}

Lastly, we give \Cref{lem:B4+} in analogy to \Cref{lem:B4}, in order to show that if $R'(\epsilon) < R$, then the conditions in \Cref{lem:B2}, \Cref{lem:B3} hold. 

\begin{lemma}\label{lem:B4+}
If $R' < R$, then we have $\lambda_{\min}(\bm{H}(t)) \geq \frac{\lambda_0}{2}$, for all $r \in [m]$, $\| w_r(t) - w_r(0)\|_2 \leq R'$ and $L(t) \leq t^{-\alpha} A(\alpha,\gamma,\lambda_0) L(0) $, for $4< \alpha \leq \frac{2d}{3}$ and $\gamma > 6$.
\end{lemma}

The width requirement for $R' < R$ is equivalent to $\frac{ 2\epsilon^{2 - \alpha/2} }{\alpha - 4} \sqrt{\frac{nC(\alpha,\gamma){L}(0)}{m  (\lambda_0/2)^{\alpha/2}(\gamma - \alpha/2 + 1)^2}  }  + \sqrt{2L(0)} \epsilon < O(\frac{\delta \lambda_0}{n^2}) $ for a sufficiently small fixed $\epsilon > 0$. We show the details in \Cref{width NAG} that it suffices to use sufficiently large $m=\Omega\left(\frac{n^{5\alpha/2 - 4}}{\delta^{3\alpha/2 - 3} \lambda_0^{3\alpha/2 - 2}}\right)$.  We note this width lower bound is larger than the width required by GD  in \cite{du2018gradient} and our HB analysis, suggesting smaller $\alpha$ or $\gamma$ may be preferred for NAG (see \Cref{fig:NAG width}), since the width requirement is increasing in $\alpha$.
\begin{figure}[!htb]
    \centering
\includegraphics[width=8cm]{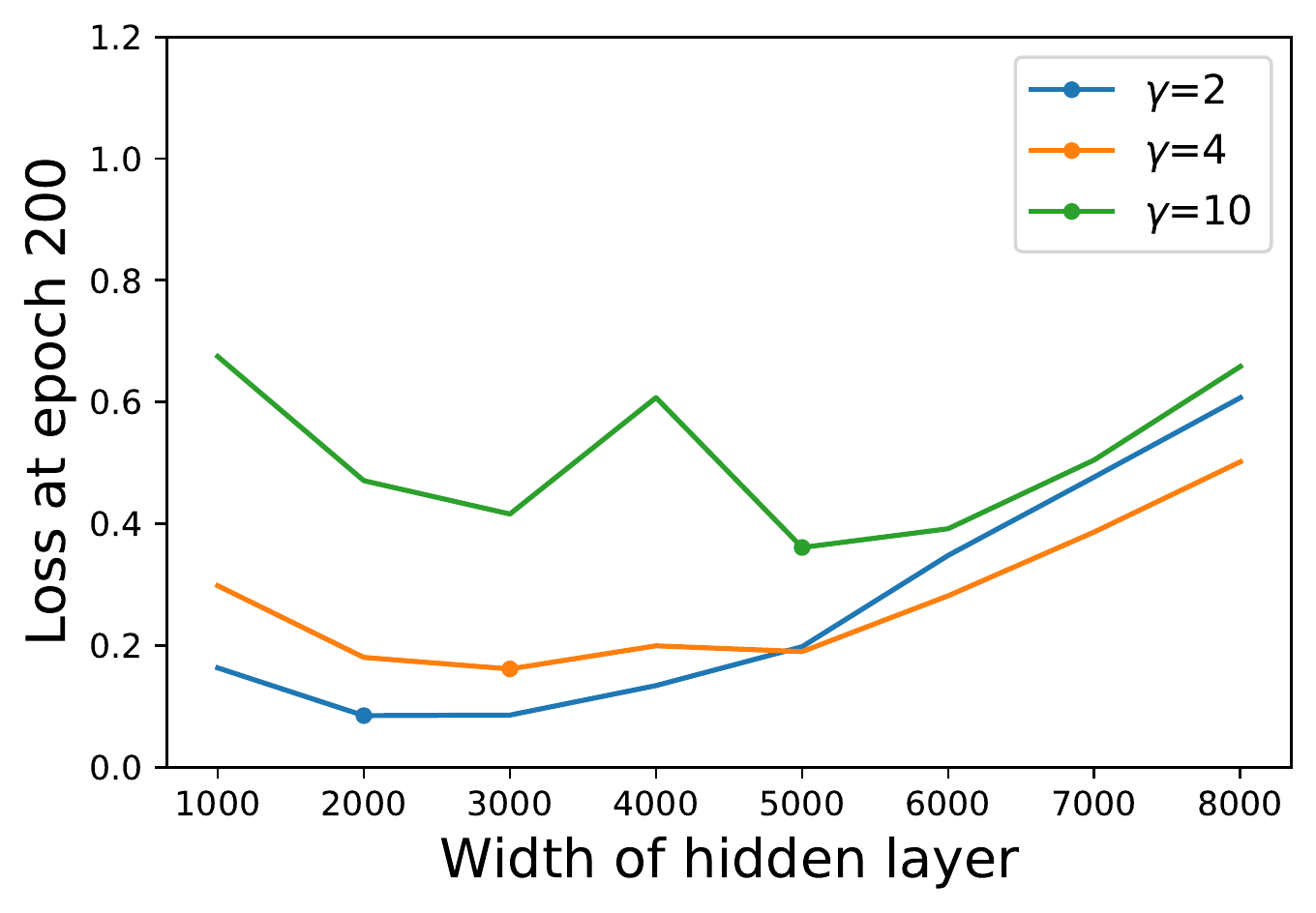}
   \caption{MSE loss with different width of the hidden layer on MNIST. We observe that NAG with larger $\gamma$ may require wider layer to converge at optimal rate. Here we fix the learning rate at $5\times 10^{-5}$ and the dots represent minimum MSE for each case. The loss is averaged of 10 runs. Details of experiments are left in \Cref{exp:nag}.}
    \label{fig:NAG width}
\end{figure}

\subsection{Nesterov Accelerated Gradient - SC}
\label{sec:NAGSC}
In this section, we demonstrate that the Nesterov Accelerated Gradient (NAG) method in a particular form (with time-independent momentum) can enjoy linear convergence. Consider the following optimization algorithms, known as the NAG-SC,
\begin{align*}
\w_r(k+1) &=\v(k+1)+\frac{1-\sqrt{\lambda_0 \eta/2}}{1+\sqrt{\lambda_0 \eta/2}}\left(\v(k+1)-\v(k)\right)
\\
\v(k+1) &=\w_r(k)-\eta \frac{\partial L}{\partial\w_r(k)}
\end{align*}
or equivalently
\begin{align*}
\w_r(k+1)=&\w_r(k)+\frac{1-\sqrt{\lambda_0\eta/2}}{1+\sqrt{\lambda_0\eta/2}}\left(\w_r(k)-\w_r(k-1)\right)
\\
&-\eta\frac{\partial L}{\partial \w_r(k)}-\frac{1-\sqrt{\lambda_0\eta/2}}{1+\sqrt{\lambda_0\eta/2}} \cdot \eta\left(\frac{\partial L}{\partial \w_r(k)}-\frac{\partial L}{\partial \w_r(k-1)}\right).
\end{align*}
We can obtain the weight dynamics as in \cite{wilson2016lyapunov,shi2018understanding},
\begin{align}
\ddot{\bm w_r}(t)+\sqrt{2\lambda_0}\dot{\bm w_r}(t)+\frac{\partial L(\W(t),\bm a)}{\partial \bm w_r(t)}=0.
\end{align}
Consequently the error dynamics is
$$ \ddot\Delta(t)+\sqrt{2\lambda_0}\dot\Delta(t)+\frac{\partial\hat L}{\partial\Delta(t)}\overset{\text{a.s.}}{=}0$$
and this is exactly the same as \eqref{eq:deltadyn}. Therefore we call \Cref{thm:HBF linear rate} to claim the same linear convergence rate for NAG-SC.

\begin{corollary}
Suppose we set the width of the hidden layer $m=\Omega\left(\frac{n^6}{\delta^3\lambda_0^4}\right)$ and we train with NAG-SC. If we i.i.d. initialize $\bm{w}_{r} \sim \mathcal{N}(\mathbf{0}, \mathbf{I}), a_{r} \sim \operatorname{unif}\{-1,1\} $ for $ r \in[m]$, then with high
probability at least $ 1-\delta $ over the initialization, we have
\begin{align*}
 L(t) &\leq\frac{4}{\lambda_0}\exp \left(-\sqrt{\lambda_0/2}\cdot t\right)\hat L(0)
\end{align*}
\end{corollary}

However, this form of NAG is difficult to implement in reality because $\lambda_0$ is not known and cannot be used as coefficients in the optimization algorithms. In comparison, we claim that the generalized NAG in \Cref{thm:Nesterov rate} can be employed without the knowledge of the strong convexity coefficient of $\hat L$.

\section{Notable Extensions}

The limiting ODE of different optimization algorithms gives rise to a general framework of analyzing the strongly-convex pseudo-loss $\hat L$. We can further extend the analysis in some major directions: higher order ODE, higher order optimization algorithms and deeper neural networks.

\subsection{Higher Order Momentum}
We consider the convergence behavior of higher order ODE, resulting from using multiple momentums in GD, as an extension of HB. We define the higher order momentum (HM$[J]$) by
\begin{equation} \label{higher order momentum: updating rule}
\bm w_r(k+1)=\bm w_r(k)-\eta\frac{\partial L(\W(k),\bm a)}{\partial \bm w_r(k)}+\sum_{j \in [J]}\beta_j(\bm w_r(k-j+1)-\bm w_r(k-j))
\end{equation}
which reduces to HB when $J=1$ and to GD when $J=0$. We remark that the weight updating rule can be implemented efficiently by introducing intermediate variables instead of using a single variable $w_r$ (this can be found in \Cref{app:efficient HM}). Consequently, the gradient flow is a $(J+1)$-th order ODE,
 \begin{align}
\frac{d^{J+1}}{dt^{J+1}}\w_r(t)+\sum_{j\in[J]}b_j\frac{d^j}{dt^j}\w_r(t)+\frac{\partial L(\W(t),\bm a)}{\partial \bm w_r(t)}=0.
\label{eq:HM old weight dynamics}
\end{align}
By similar argument as in \eqref{eq:second derivative zero}, we have the higher order derivative $\frac{\partial^j \f}{\partial \W^j}\overset{\textup{a.s.}}{=}0$ for $j>1$, hence the error dynamics follows the same form as the weight dynamics, except $L$ is replaced by $\hat L$. To see this, we have $\frac{\partial^j \f}{\partial t^j}\overset{\textup{a.s.}}{=}\sum_r\frac{\partial \f}{\partial \w_r}\frac{d^{j}\w_r(t)}{dt^{j}}$ and
 \begin{align}
\frac{d^{J+1}}{dt^{J+1}}\Delta(t)+\sum_{j\in[J]}b_j\frac{d^j}{dt^j}\Delta(t)+\frac{\partial \hat L}{\partial \Delta(t)}=0.
\label{HM:old error dynamics}
\end{align}

To directly analyze this high order ODE via the Lyapunov function is difficult. For the special case of $J=2$, a second momentum has been shown to further accelerate the convergence than the first momentum \cite{pearlmutter1992gradient}, when the loss is a positive quadratic form. Empirically, we observe in \Cref{fig:HM} that, although $L$ is not convex nor quadratic, employing the second momentum leads to faster convergence as well. 

\begin{figure}[!htb]
\centering
\includegraphics[width=8cm]{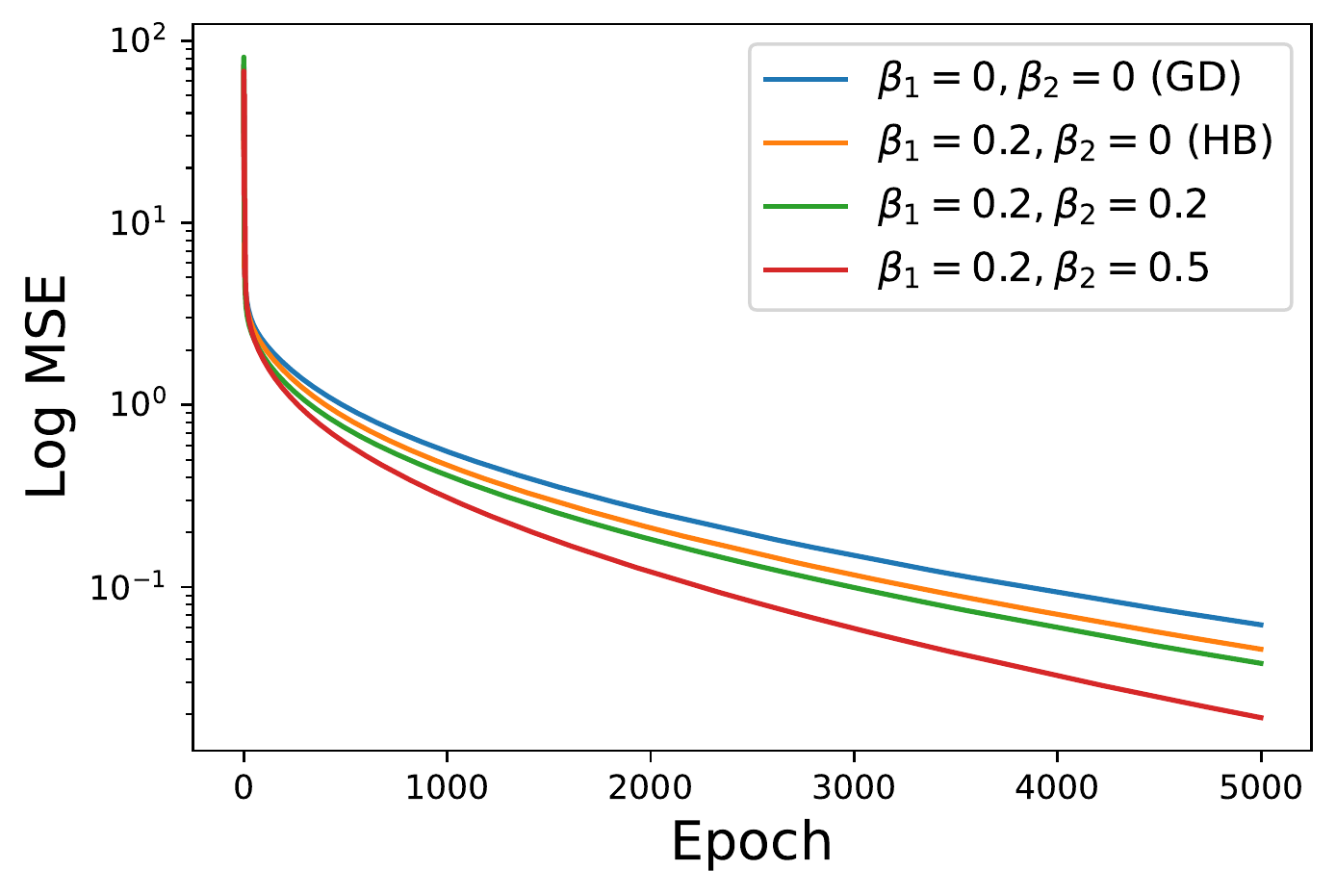}
 \caption{Logarithm MSE for optimization algorithms with different number of momentums on MNIST. We observe that the higher momentum method allows faster convergence rate than GD and HB. Details of experiments are left in \Cref{exp:hm}.} 
\label{fig:HM}
 \end{figure}

\subsection{Newton's Method}
In this section we study the higher order optimization algorithms that use the Hessian information. By the definition of Newton-Ralphson method, we have
\begin{align*}
\bm w_r(k+1)&=\bm w_r(k)-{\nabla^2 L(\W(k),\bm a)}^{-1} \nabla L(\W(k),\bm a)
\end{align*}
in which $\nabla^2 L=\frac{\partial^2 L}{\partial \bm w_r^2}
$ and $\nabla L=\frac{\partial L}{\partial \bm w_r}$. We observe by the chain rule (c.f. \eqref{eq:second derivative zero}) that
\begin{align*}
\frac{\partial^2 L}{\partial \bm w_r^2}
=\left\langle\frac{\partial \f}{\partial \w_r},\frac{\partial \f}{\partial \w_r}\right\rangle+\frac{\partial^2 \f}{\partial \bm w_r^2}(\f-\y)
\overset{\text{a.s.}}{=}\left\langle\frac{\partial \f}{\partial \bm w_r},\frac{\partial \f}{\partial \bm w_r}\right\rangle\in\R^{p\times p}.
\end{align*}
We highlight that the Hessian matrix is assumed to be invertible for Newton-Ralphson method to apply. In other words, there exists a constant $\lambda_r>0$ which lower bounds $\lambda_{\min}\left(\big(\frac{\partial \f}{\partial \bm w_r}\big)^\top\frac{\partial \f}{\partial \bm w_r}\right)$. This assumption is equivalent to $\frac{\partial \f}{\partial\w_r(t)}$ having linearly independent columns, which requires $n>p$. The corresponding gradient flow is then a first order ODE
\begin{align}
\dot\w_r&=-\left(\left(\frac{\partial \f}{\partial\w_r}\right)^\top\frac{\partial \f}{\partial \w_r}\right)^{-1}\left(\frac{\partial \f}{\partial\w_r}\right)^\top(\f-\y)
\label{eq:newton weight dynamics}
\end{align}
The error dynamics can be obtained by left multiplying $\frac{\partial \f}{\partial\w_r}$ and summing over $r$. Denoting $\nabla f_r(t):=\frac{\partial \f}{\partial\w_r(t)}\in \R^{n\times p}$, we have
\begin{align}
\dot{\Delta}(t)\overset{\text{a.s.}}{=}-\sum_r \nabla f_r\left(\nabla f_r^\top \nabla f_r\right)^{-1}\nabla f_r^\top\Delta(t)
\label{eq:newton error dynamics}
\end{align}
We can define the pseudo-loss as
$$\hat L=\frac{1}{2}\Delta^\top\Big(\sum_r\nabla f_r\left(\nabla f_r^\top\nabla f_r\right)^{-1}\nabla f_r^\top\Big)\Delta$$
so that the error dynamic is simply $\dot\Delta=-\frac{\partial \hat{L}}{\partial\Delta}.$

We state our result about the convergence of the Newton-Ralphson method.
\begin{proposition}\label{thm:Newton linear rate}
Suppose we train with Newton-Ralphson method and $\lambda_{\min} (\bm{H}(t)) \geq \frac{\lambda_0}{2}$, then we have
\begin{align*}
 L(t) &\leq\exp \left(-\frac{\lambda_0 mt}{n}\right)L(0)
\end{align*}
\end{proposition}
We highlight that, \Cref{thm:Newton linear rate} gives the Newton-Ralphson method a significantly faster linear convergence rate than GD if $m/n\gg 1$. This implies the advantage of using the second-order optimization algorithms over the first-order ones, if we set aside the concerns on the complexity and memory incurred by the Hessian information.

\subsection{Multi-layer Training}
\label{sec:multi_layer}
Though we do not dive into the detailed proof of extending our main theorem to multi-layer neural network, including training both layers of the two-layer neural network and deep neural networks, we provide a sketch proof direction as follows: \\\\
On high level, we can write the training dynamics for the deep fully-connected neural networks as \begin{gather*}
\x_r^{(l)}=\frac{1}{\sqrt{m}}\sigma\left((\w_r^{(l)})^\top\x^{(l-1)}\right)\\
f(\x;\W,\a)=\a^\top\x^{\mathcal{L}}
\end{gather*}
in which the layer index $l\in[\mathcal{L}-1]$, input $\x^{(0)}=\x$, the post-activation $\x_r^{(l)}$ corresponds to the $r$-th hidden neuron in the $l$-th layer and each layer has $m$ hidden neurons, i.e. $r\in [m]$. We denote the union of all weights as $\W=\{\w_r^{(1)},\cdots,\w_r^{(\mathcal{L}-1)},\a\}$.
Similar to \eqref{eq:weight dynamics}-\eqref{eq:error dynamics}, we have
\begin{align*}
\ddot{\Delta}(t)+b(t)\dot{\Delta}(t) +\H(t)\Delta(t)=0
\end{align*}
where $b(t)$ depends on different optimization algorithms. Here the NTK is
$\H_{multi}(t)=\sum_l \H_l(t)+\H_a(t)$ with $\H_l(t)=\sum_r\frac{\partial \f(t)}{\partial \bm{w}_{r}^{(l)}(t)}\left(\frac{\partial \f(t)}{\partial \bm{w}_{r}^{(l)}(t)}\right)^\top$ and $\H_a(t)=\frac{\partial \f(t)}{\partial \a(t)}\left(\frac{\partial \f(t)}{\partial \a(t)}\right)^\top$. Clearly $\H_l$ and $\H_a$ are positive semi-definite. Therefore it suffices to show there exists one $l\in[\mathcal{L}-1]$ that is positive definite. In our two-layer analysis, we show the Gram matrix corresponding to the last (and the first) hidden layer is positive definite. We note that \cite{dugradient} has also shown $\H_{\mathcal{L}-1}$ is positive definite for GD and for deep fully-connected neural networks. We believe similar analysis applies to other optimization algorithms such as HB and NAG.


\section{Discussion}
In this paper, we extend the convergence analysis of over-parameterized neural networks to different accelerated optimization algorithms, including the Heavy Ball method (HB) and the Nesterov accelerated gradient descent (NAG). Our analysis is based on the neural tangent kernel (NTK) which characterizes the training dynamics as ODEs known as the gradient flows. We observe that, for piecewise linear activation functions (e.g. ReLU, Leaky ReLU \cite{maas2013rectifier} and maxout \cite{goodfellow2013maxout, he2015delving}), the weight dynamics takes the same form as the error dynamics (for example, \eqref{eq:heavy ball flow} and \eqref{eq:heavy ball error flow}), with the only difference lying in the losses. To see this, we recall $\frac{\partial^2 \f}{\partial\w_r\partial\w_l}\overset{\text{a.s.}}{=}0$ for \textit{piecewise linear activations}. In particular, the loss in the error dynamics is strongly-convex, leading to linear convergence of HB (and GD) and polynomial decay or linear convergence of NAG. We emphasize that by constructing the strongly-convex loss $\hat L$, we can easily borrow the rich results from the convex optimization world to analyze the convergence of neural networks on non-convex loss. 
We remark that instead of using the Gronwall inequality, as in the case of GD \cite{du2018gradient}, we use the Lyapunov function, which is a common and traditional tool in solving ODE with convex loss. In fact, the Gronwall inequality generally does not work on second or higher order ODE. 

A major extension of our one-layer training process will be to train all layers in deep neural networks, which has been well-established for GD in \cite{allen2019convergence,dugradient} for various types of neural networks such as RNN, CNN, ResNet, Graph neural networks (GNN) and Generative Adversarial networks (GAN). Since the main focus of this work is on the optimization algorithms instead of neural network architectures, we do not over-complicate our proof by exploring deeper learning (including two-layer training) but leave the discussion in \Cref{sec:multi_layer}. We also emphasize that the error dynamics of GD/HB/NAG, i.e. the limiting ODEs, are indeed compatible to arbitrary neural networks such as CNN and ResNet. To extend our analysis only requires a different route to guarantee $\lambda_{\min}(\H(t))>0$. For similar reasons, we claim the discrete time convergence analysis of the HB and NAG is similar to their continuous time counter-parts and can be derived within our framework of proofs, but we do not try to present the rigorous theorems in this paper.

We pause to discuss the width requirement in NTK regime. Unfortunately, existing works generally require unrealistic width (see the summary in \cite{zou2019improved}): $n^{24}$ is required in \cite{allen2019convergence} and $n^8$ in \cite{zou2019improved}, both comparable to our $\frac{n^6}{\lambda_0(n)^4}$. Hence we do not attempt to reduce our width with any complicated method. Nevertheless, all our experiments use discrete gradient descents and realistic widths, ranging from $10^3$
to $10^4$ hidden neurons. This is detailed in \Cref{app:experiment details}.

Another important future direction is to study the adaptive optimizers such as

Our analysis may further extend to more optimization algorithms. Unlike GD, HB and NAG, many popular optimization algorithms adopt adaptive learning rates and demonstrate impressive performance. Examples include AdaGrad \cite{duchi2011adaptive}, AdaDelta \cite{zeiler2012adadelta}, RMSprop \cite{hinton2012neural}, Adam \cite{kingma2014adam}, DIN\cite{alvarez2002second} and especially their variants with momentums and mini-batches. These optimizers are expected to converge faster than those analyzed in this work, but they may require a different framework to analyze. Notably, optimization algorithms with adaptive learning rates can correspond to a system of limiting ODEs, instead of a single ODE. In \cite{da2020general}, the following systems of ODEs characterize the evolution of Adam, AdaFom, HB and NAG with different functions $(h,\gamma,p,q)$ on the left and the evolution of AdaGrad and RMSprop without a momentum term $\m_r(t)$ on the right:
\begin{align*}
\left\{\begin{aligned} \dot{\w_r}(t) &=-\m_r(t) / \sqrt{\v_r(t)+\varepsilon} \\ \dot{\m_r}(t) &=h(t) \frac{\partial L}{\partial \w_r(t)}-\gamma(t) \m_r(t) \\ \dot{\v_r}(t) &=p(t)\left[\frac{\partial L}{\partial \w_r(t)}\right]^{2}-q(t) \v_r(t) \end{aligned}\right. 
&\quad\quad\text{or}\quad\quad
\left\{\begin{aligned} \dot{\w_r}(t) &=-\frac{\partial L}{\partial \w_r(t)} / \sqrt{\v_r(t)+\varepsilon} \\ \dot{\v_r}(t) &=p(t)\left[\frac{\partial L}{\partial \w_r(t)}\right]^{2}-q(t) \v_r(t) \end{aligned}\right. 
\end{align*}
Interestingly, we observe that the error dynamics again follows a similar form of the weight dynamics. Suppose we consider the memory term $\v_r(t)\in \R^p$ with identical entries across $r$, i.e. all weights share the same adaptive learning rate. These optimizers are still adaptive to the training dynamics though not in a per-parameter sense. Then we can treat $\v_r(t)$ as a scalar $v(t)$. By multiplying $\frac{\partial \f}{\partial \bm{w}_r}$ and sum over $r$, the error dynamics for adaptive optimization algorithms will be
\begin{align*}
\left\{\begin{aligned} 
\dot{\Delta}(t) &=-\Q(t) / \sqrt{v(t)+\varepsilon} 
\\ 
\dot{\Q}(t) &=h(t) \frac{\partial \hat L}{\partial \Delta(t)}-\gamma(t) \Q(t)
\end{aligned}\right. 
&\quad\quad\text{or}\quad\quad
\begin{aligned} 
\dot{\Delta}(t) &=-\frac{\partial \hat L}{\partial \Delta(t)}/ \sqrt{v(t)+\varepsilon} 
\end{aligned}
\end{align*}
in which $\Q(t):=\sum_r\frac{\partial\f}{\partial\w_r}\m_r(t)$. It would be desirable to analyze these adaptive optimization algorithms in future works, though the convergence analysis can be more difficult than in this work and require more advanced tools in Lyapunov function and ODE theory.

\newpage
\medskip
\printbibliography

\clearpage
\appendix

\section{Deriving Width from $R'<R$}
In this section we derive the width requirement for our main theorems. We slightly abuse the notation of $\alpha$ here: For HB analysis, $\alpha:=\sqrt{\lambda_0/2}$. For NAG analysis, $\alpha$ is a constant in $[4,2d/3]$.

\subsection{Width Requirement in \Cref{thm:HBF linear rate} (HB)}
\label{width HB}
\paragraph{\Cref{thm:HBF linear rate}}
To guarantee $R'<R$, we need $\sqrt{\frac{32n\hat L(0)}{9m\alpha^6}}<O\left(\frac{\delta\lambda_0}{n^2}\right)$ which is equivalent to $m=\Omega\left(\frac{n^5}{\delta^2\lambda_0^5}\cdot \hat L(0)\right)$. Since $\lambda_{\min}(\H(0))>\lambda_0/2$, we have $\hat L>\frac{\lambda_0}{2}L(0)$ and we can bound $L(0)$ using Markov's inequality: with probability at least $ 1-\delta$, we obtain $L(0)=O\left(\frac{n}{\delta}\right)$, because
$$
\mathbb{E}[L(0)]=\sum_{i=1}^{n}\left(y_{i}^{2}+y_{i} \mathbb{E}\left[f\left(\mathbf{W}(0), \mathbf{a}, \mathbf{x}_{i}\right)\right]+\mathbb{E}\left[f\left(\mathbf{W}(0), \mathbf{a}, \mathbf{x}_{i}\right)^{2}\right]\right)=\sum_{i=1}^{n}\left(y_{i}^{2}+1\right)=O(n). 
$$
Therefore
$$m=\Omega\left(\frac{n^4}{\delta^2\lambda_0^4}\cdot \frac{n}{\lambda_0}\cdot\frac{n\lambda_0}{\delta}\right)=\Omega\left(\frac{n^6}{\delta^3\lambda_0^4}\right).$$

\subsection{Width Requirement in \Cref{thm:Nesterov rate} (NAG)} 
\label{width NAG}
To guarantee $R'<R$, we need $ \frac{ 2\epsilon^{2 - \alpha/2} }{\alpha - 4} \sqrt{\frac{nC(\alpha,\gamma){L}(0)}{m  (\lambda_0/2)^{\frac{\alpha}{2}}(\gamma - \alpha/2 + 1)^2}  }  + \sqrt{2L(0)} \epsilon < O(\frac{\delta \lambda_0}{n^2}) $. Choose small enough $\epsilon$ such that $\sqrt{2L(0)}\epsilon < R/2 $. This is equivalent to set $\epsilon = O(\frac{\delta^{3/2}\lambda_0}{n^{5/2}})$ since $L(0)=O(n/\delta)$. Then we need $\frac{ 2\epsilon^{2 - \alpha/2} }{\alpha - 4} \sqrt{\frac{nC(\alpha,\gamma){L}(0)}{m  (\lambda_0/2)^{\frac{\alpha}{2}}(\gamma - \alpha/2 + 1)^2}  }<R/2$. Ignoring constant terms, this is equivalent to $ \epsilon^{2 - \alpha/2} \sqrt{\frac{n{L}(0)}{m  \lambda_0^{\alpha/2}}  }< O(\frac{\delta \lambda_0}{n^2})$ and further to $m = \Omega\left( \frac{\epsilon^{4-\alpha} n^5}{\lambda_0^{\alpha/2 + 3 }\delta^2 } L(0) \right) $. Similar to the analysis in the previous section, we have $L(0)=O(n/\delta)$ and therefore
$$m=\Omega\left(\frac{n^{5\alpha/2 - 4}}{\delta^{3\alpha/2 - 3} \lambda_0^{3\alpha/2 - 2}}\right).$$

Notice that as $\alpha\to 4$, the width requirement tends to the same width as for GD and HB. Moreover, the width is increasing in $\alpha$.

\section{Memory-efficient Implementation of HM Method}
\label{app:efficient HM}
Here we derive the iterative term for second-order momentum $(J=2)$ for implementation similar to \cite{sutskever2013importance}.  Recall that for the weight dynamics in the single-variable form is
\begin{align*}
\bm w_r(t+1)=\bm w_r(t)-\eta\frac{\partial L(\bm w_r(t),\bm a)}{\partial \bm w_r(t)}
&+\beta_1(\bm w_r(t)-\bm w_r(t-1))+\beta_2(\bm w_r(t-1)-\bm w_r(t-2))
\end{align*}
Denoting $\frac{\partial L(\bm w_r(t),\bm a)}{\partial \bm w_r(t)}$ as $g(t)$ gives the following term:
\begin{align*}
\bm w_r(t+1)=\bm w_r(t)-\eta g(t)+\beta_1(\bm w_r(t)-\bm w_r(t-1))+\beta_2(\bm w_r(t-1)-\bm w_r(t-2))
\end{align*}
Our implementation will be more efficient than directly implementing the single-variable update above. Introducing two new variables $\bm v, \bm q$ and given an objective loss function $L$ to be minimized, classical momentum is given by:
\begin{align*}
\bm w_r(t+1) &= \bm w_r(t) - \eta \bm v(t+1)\\
\bm v(t+1) &= a\bm v(t) +\bm q(t+1)\\
\bm q(t+1) &= b\bm q(t)+g(t)
\end{align*}
where $a = \frac{1}{2}(\beta_1+\sqrt{\beta_1^2+4\beta_2})$ and $b= \frac{1}{2}(\beta_1+\sqrt{\beta_1^2-4\beta_2})$. Note that when $\beta_2=0$, the above momentum is the momentum of heavy ball. And when $\beta_1=0, \beta_2=0$, it can recover gradient descent.

\section{Proof of \Cref{lem:B4} and \Cref{lem:B4+}}
\label{3442}
This proof is almost identical to \cite[Lemma 3.4]{du2018gradient} and is adapted to here for the completeness of our paper.
\\\\
Suppose the conclusion does not hold for at time $t$. Then we split the analysis into two cases: (1) there exists $r\in [m]$ such that $\|\w_r(t)-\w_r(0)\|_2\geq R^\prime$; (2) $L(t)\geq \frac{4}{\lambda_0}\exp(-\sqrt{\lambda_0/2}t)\hat L(0)$.\\
For case (1), if we are proving \Cref{lem:B3}
or $L(t)\leq A(\alpha, d, \lambda_0)t^{-\alpha}\hat L(0)$ if we are proving \Cref{lem:B3+}. 
We know there exists $s\leq t$ such that $\lambda_{\min}(\H(s))<\frac{1}{2}\lambda_0$. However, \Cref{lem:B2} shows that there exists a $t_0$ with
\begin{align*}
t_{0}=\inf \left\{t>0: \max _{r \in[m]}\left\|\mathbf{w}_{r}(t)-\mathbf{w}_{r}(0)\right\|_{2}^{2} \geq R\right\}.
\end{align*} 
Consequently, there exists $r\in [m]$ that  $\|\w_r(t_0)-\w_r(0)\|_2=R$. Applying \Cref{lem:B2}, we know that $\H(t')\geq \frac{1}{2}\lambda_0$ for $t' \leq t_0$. However, by \Cref{lem:B3}/\Cref{lem:B3+} we know $\|\w_r(t_0)-\w_(t)\|_2<R^\prime<R$, a contradiction.\\
For case (2), at time $t$, $\lambda_{\min}(\H(t))<\frac{1}{2}\lambda_0$ we know there exists
\begin{align*}
t_{0}=\inf \left\{t\geq 0: \max _{r \in[m]}\left\|\mathbf{w}_{r}(t)-\mathbf{w}_{r}(0)\right\|_{2}^{2} \geq R\right\}.
\end{align*} 
The rest of the proof is the same as the previous case.

\section{Experiment}
\label{app:experiment details}
The MNIST dataset \cite{lecun2010mnist} is a dataset of handwritten digits ranging from 0 to 9. It contains 60,000 training images and 10,000 test images. Each image is in $28 \times 28$ gray-scale. We train the two-layer neural networks with random initialization: $\w_r$ (the weight in the first layer) is initialized with distribution $ \mathcal{N}(\mathbf{0}, \mathbf{I})$ and $a_{r}$ (the weight in the second layer) is initialized with distribution $ \operatorname{unif}\{-1,1\}$. We randomly take 600 subsamples out of the original training dataset, in order to run full-batch gradient descents within limited GPU memory.
\subsection{Experiment of Heavy Ball (HB)}
\label{exp:hb}
This subsection corresponds to  \Cref{fig:HB}. The neural network is trained in 5000 steps with hidden neurons in the two-layer neural network is 3000. We use three different momentum term $\beta=0, 0.2, 0.5, 0.8$ with learning rate $10^{-4}$. When $\beta=0$, the HB reduces to GD. We take the log of MSE in the y-axis. The constant slope in figures depicts the linear convergence rate. 
\subsection{Experiment of Nesterov Accelerated Gradient (NAG)}
\label{exp:nag}
The basic setting of \Cref{fig:NAG} is the same as \Cref{exp:hb} except we replace the momentum term with "nesterov term" $\gamma$, which are $2,3,6,8$ respectively. We take log on both the time/step size as well as the MSE loss to show the polynomial decay of NAG. In \Cref{fig:NAG width}, the learning rate is fixed as $5*10^{-5}$ while width ranges from $1000$ to $9000$. The experiment is run for ten times using different seeds for random initialization and we record the loss at epoch 200 for every $\gamma$ and width.

\subsection{Experiment of Higher Order Momentum (HM)}
\label{exp:hm}
In \Cref{fig:HM}, we take $\beta_1=0, 0.2$ and $\beta_2$ varies from $0, 0.2, 0.5$. Here $\beta_2=0$ recovers the Heavy Ball (HB) and $\beta_1=0$ recovers the Gradient Descent (GD). The total number of epochs, number of hidden neurons and learning rate are the same as \Cref{exp:hb}. We take log of MSE in the y-axis along with epochs in the x-axis.

\section{Additional Technical Details}
\subsection{Convergence of NAG in Strongly Convex Case \cite{su2014differential}}
\label{app:su nesterov}
The following result is used in the proof of \Cref{lem:B3+}. Here $\hat L^\star$ means the global minimum of $\hat L$ and $\Delta^\star$ is the global minimizer, both of which are zero.
\begin{theorem}[Theorem 8 in \cite{su2014differential}]
For any $ \hat L $ being $\mu$–strongly convex on $\R^n$ with continuous gradients, if $ 2 \leq \alpha \leq 2 \gamma / 3 $ we get from the dynamics \eqref{eq:NAG error dynamics}  that
\begin{align}
 \hat L(\Delta(t))-\hat L^{\star} \leq \frac{C(\alpha,\gamma)\left\|\Delta(0)-\Delta^{\star}\right\|^{2}}{\mu^{\frac{\alpha}{2}-1} t^{\alpha}}
\end{align}
for any $ t>0 . $ Above, the constant $ C(\alpha,\gamma) $ only depends on $ \alpha $ and $ \gamma . $
\end{theorem}
We do not copy the proof in \cite{su2014differential} here but note that the above theorem was shown via a careful analysis of the Lyapunov function
$$
  \mathcal{E}(t ; \alpha)=t^{\alpha}\left(\hat L(\Delta(t))-\hat L^{\star}\right)+\frac{(2 \gamma-\alpha)^{2} t^{\alpha-2}}{8}\left\|\Delta(t)+\frac{2 t}{2 \gamma-\alpha} \dot{\Delta}(t)-\Delta^{\star}\right\|^{2}.
$$

\subsection{Convergence of Newton's Method}
\begin{proof}[Proof of \Cref{thm:Newton linear rate}]
In order to guarantee the linear convergence, we need the sum of matrices, which we refer to as $\H_{\textup{newton}}(t)=\sum_r \nabla f_r\left(\nabla f_r^\top \nabla f_r\right)^{-1}\nabla f_r^\top\in\R^{n\times n}$, to be positive definite for all $t$. Notice $\H_{\textup{newton}}(t)$ is clearly positive semi-definite and we show $\v^\top\left[\sum_r \nabla f_r(\nabla f_r^\top \nabla f_r)^{-1}\nabla f_r^\top\right]\v>0$, for any non-zero vector $\bm v$. In fact, we have
\begin{align*}
\v^\top\left[\sum_r \nabla f_r(\nabla f_r^\top \nabla f_r)^{-1}\nabla f_r^\top\right]\v
\geq \sum_r \lambda_{\min}\left((\nabla f_r^\top \nabla f_r)^{-1} \right)\|\nabla f_r^\top\v\|^2.
\end{align*}
Since the matrix $\nabla f_r^\top \nabla f_r$ is non-singular, all its eigenvalues are positive. Then the non-zero eigenvalues of $(\nabla f_r^\top \nabla f_r)^{-1}$ are equal to the inverse of eigenvalues of $\nabla f_r^\top \nabla f_r$. We obtain
\begin{align*}
&\v^\top\left[\sum_r \nabla f_r(\nabla f_r^\top \nabla f_r)^{-1}\nabla f_r^\top\right]\v
\geq \sum_r \frac{1}{\lambda_{\max}\left(\nabla f_r^\top \nabla f_r\right)}\|\nabla f_r^\top\v\|^2
\\
=&\sum_r \frac{1}{\lambda_{\max}\left(\nabla f_r \nabla f_r^\top\right)}\|\nabla f_r^\top\v\|^2
\geq\frac{1}{\max_r\lambda_{\max}\left(\nabla f_r \nabla f_r^\top\right)}\sum_r\|\nabla f_r^\top\v\|^2
\\
=&\frac{1}{\max_r\lambda_{\max}\left(\nabla f_r \nabla f_r^\top\right)}\v^\top\H(t)\v
\geq \frac{1}{\max_r\lambda_{\max}\left(\nabla f_r \nabla f_r^\top\right)}\frac{\lambda_0}{2}\|\v\|_2^2
\end{align*}
which is strictly positive and $\H(t)$ is as defined in \eqref{eq:NTK matrix}. Here the first equality holds since the non-zero eigenvalues of $\nabla f_r \nabla f_r^\top$ are the same as the non-zero eigenvalues of $\nabla f_r^\top \nabla f_r$. 

We now seek a time-independent lower bound for the smallest eigenvalue of $\H_{\textup{newton}}(t)$. In particular, since all eigenvalues of $\nabla f_r \nabla f_r^\top$ are non-negative, we see from \eqref{eq:H(t) definition} that,
\begin{align*}
\lambda_{\max}(\nabla f_r \nabla f_r^\top)\leq \text{tr}(\nabla f_r \nabla f_r^\top)=\sum_{i=1}^n \frac{1}{m}\I\left(\bm w_r(t)^\top \bm x_i\geq 0\right)\leq \frac{n}{m}.
\end{align*}
By the error dynamics \eqref{eq:newton error dynamics}, we can derive 
\begin{align}
\frac{d}{dt}(\Delta^\top\Delta)=2\Delta^\top\dot{\Delta}=-2\Delta^\top\H_{\textup{newton}}(t)\Delta\leq -\frac{\lambda_0 m}{n}\Delta^\top\Delta
\end{align}
which leads to $L(t)\leq\exp\left(-\frac{\lambda_0 mt}{n}\right)L(0)$.
\end{proof}

\end{document}